\documentclass[reqno,11pt]{amsart}  
\usepackage[utf8]{inputenc}

\usepackage[colorinlistoftodos]{todonotes}
\usepackage{pgfplots} 
\usepackage{subfigure}
\usepackage{amsmath}
\usepackage{amssymb} 
\usepackage{mathtools}   
\usepackage{graphicx}
\usepackage{graphics}  
\usepackage{color}
\usepackage{verbatim} 
\usepackage[normalem]{ulem} 
\usepackage[mathscr]{eucal}
\usepackage{upgreek}
\usepackage{enumerate} 
\usepackage{cite}
\usepackage{amsmath}
\usepackage{amsfonts}
\usepackage{amssymb}
\usepackage{bbm}
\usepackage{cancel}
\usepackage{graphicx}
\usepackage{multicol}
\usepackage[utf8]{inputenc}
\usepackage{framed}
\usepackage{setspace}
\usepackage{amsthm}
\usepackage{hyperref}
\usepackage{caption}
\usepackage{subcaption}
\usepackage{bbm}
\usepackage{graphicx}
\usepackage{tikz}
\usepackage[titletoc]{appendix}
\numberwithin{equation}{section}  

\usepackage[refpage,noprefix]{nomencl}
\usepackage{nomencl}

\newcommand{\ssup}[1] {{\scriptscriptstyle{{#1}}}}
\newcommand{\gk}[1]{\left\{#1\right\}}

\newcommand\mycom[2]{\genfrac{}{}{0pt}{}{#1}{#2}}
\newcommand{\ek}[1]{\left[#1\right]}
\newcommand{\rk}[1]{\left(#1\right)}

\newcommand{\hk}[1]{^{\ssup{(#1)}}}
\newcommand{\abs}[1]{\left| #1 \right|}

\newcommand{\Bcal}   {{\mathcal B }}
 
\newcommand{\Dcal}   {{\mathcal D }}

\newcommand{\Ncal}   {{\mathcal N }}

\newcommand{\R}     {\mathbb{R}} 
 
\newcommand{\N}     {\mathbb{N}} 
\renewcommand{\P}   {\mathbb{P}} 
 
\newcommand{\E}     {\mathbb{E}}

 \newcommand{\ex}{{\rm e}} 
 \renewcommand{\d}{{\rm d}}

\newcommand{\e}{\varepsilon}
\newcommand{\1}{\mathbbm{1}}

\renewcommand{\P}{\mathbb{P}}

\newtheorem*{example}{Example}
\newtheorem{theorem}{Theorem}
\newtheorem{lemma}{Lemma}[section]
\newtheorem{definition}[lemma]{Definition}
\newtheorem{proposition}[lemma]{Proposition}

\newtheorem{assumption}{Assumption}
\newcommand{\ip}[1]{\left\langle #1\right\rangle}

\newcommand{\norm}[1]{\left|\hspace{-0.395mm}\left| #1\right|\hspace{-0.395mm}\right|}

\newcommand{\lk}[1]{\left\langle #1\right\rangle }

\newcommand{\fb}[1]{\left\langle #1\right\rangle_{\ssup{\mathrm{F}}}}
\newcommand{\ul}[1]{\underline{#1}}

\newcommand{\ux}{\ul{x}}
\newcommand{\nf}[1]{\norm{#1}_{\ssup{\mathrm{F}}}}
\newcommand{\LM}{\mathrm{M}}
\newcommand{\uZx}{Z_{\ux}}

\newcommand{\hok}[1]{^\ssup{#1}}

\title{Large Deviations of Gaussian Neural Networks with ReLU activation}
\author{Quirin Vogel$^{1,2}$}
\date{August 2025}
\pgfplotsset{compat=1.18}
\begin{document}

\begin{abstract}
    We prove a large deviation principle for deep neural networks with Gaussian weights and at most linearly growing activation functions, such as ReLU. This generalizes earlier work, in which bounded and continuous activation functions were considered. In practice, linearly growing activation functions such as ReLU are most commonly used. We furthermore simplify previous expressions for the rate function and provide a power-series expansions for the ReLU case.
\end{abstract}
\maketitle
\centerline{\textit{$^1$Ludwig-Maximilians-Universität München,
Mathematisches Institut, München, Germany}}
\centerline{\textit{$^2$Alpen-Adria-Universität Klagenfurt,
Department of Statistics, Klagenfurt, Austria}}
\bigskip


\section{Introduction and Results}
\subsection{Definition of the model}
The ability to learn complex and highly non-linear relations has made deep neural networks one of the most promising tools to achieve artificial general intelligence, see \cite{liu2015representation}. 
In general, neural networks are often parameterized by the number of \textit{layers} $L\in \N$ and the size of each layer $n_0,n_1,\ldots,n_L,n_{L+1}$ where the size of the $i$-th layer is denoted by $n_i$, $n_0$ is the dimension of the input and $n_{L+1}$ that of the output. A  neural network is considered \textit{deep} if $L>1$; otherwise, it is \textit{shallow}. Crucial for the behavior of the neural network is the choice of activation function $\sigma\colon\R\to\R$ which governs how the different layers influence each other.

We state our assumptions on the activation function $\sigma$:
\begin{assumption}\label{ass:activation}
    We assume that for some $c_+\ge 0$, $\sigma(x)= c_+ x+o(x)$ as $x\to \infty$. For $x\to -\infty$, we assume that $\sigma(x)= -c_-x+o(x)$ with $c_-\in [-c_+,0]$, i.e., $\sigma$ grows at most linearly. Furthermore, we require $\sigma$ to be measurable, almost everywhere continuous, bounded on compact sets and non-trivial, i.e., different from $0$ on a set of positive Lebesgue measure.
\end{assumption}
Note that our definition encompasses all the important activation functions from the literature:
\begin{example}\label{exam:}
Commonly used activation functions include:
    \begin{itemize}
        \item {ReLU} (\textbf{Re}ctified \textbf{L}inear \textbf{U}nit, see \cite{hahnloser2000permitted}) given by $\sigma(x)=\max\gk{x,0}$ satisfies the assumption with $c_+=1$ and $c_-=0$.
        \item The {sigmoid} function $\sigma(x)=\mathrm{sigmoid}(x)=\frac{1}{1+\ex^{-x}}$ satisfies the assumption with $c_+=c_-=0$, see \cite{gershenfeld1999nature} for the use as activation function.
        \item {Binary step} with $\sigma(x)=\1\gk{x\ge 0}$ satisfies the assumption with the same constants as sigmoid.
        \item Gaussian-Error linear unit (GELU) (see \cite{hendrycks2016gaussian}) with $\sigma(x)=x\P_{\Ncal(0,1)}(X\le x)$ satisfies the assumption with $c_+=1$ and $c_-=0$.
        \item {Swish} (see \cite{hendrycks2016gaussian}) with $\sigma(x)=x\cdot\mathrm{sigmoid}(x)$ also satisfies the assumption with $c_+=1$ and $c_-=0$.
        \item The same holds for Softplus (see \cite{glorot2011deep}, ELU (see \cite{clevert2015fast}), Mish (see \cite{misra2019mish}) and Squareplus (see \cite{barron2021squareplus}).
        \item Parametric ReLU's (see \cite{he2015delving}) satisfy the assumptions with $c_+=1$ and $c_-=a\in [0,1)$.
    \end{itemize}
    Notably, we are unaware of any scalar activation function in the literature that violates Assumption \ref{ass:activation}.
\end{example}
With our assumptions on $\sigma$ stated, we now explain how we define the Gaussian neural network.

For two sets $A,B$, we write $\LM(A,B)$ for the space of linear maps from $\R^A$ to $\R^B$. For each $M\in\LM(A,B)$, we assume it to be represented by a matrix. Furthermore, for $k\in \N$, we write $\LM(A,k)$ as shorthand for $\LM(A,[k])$ where $[k]=\gk{1,\ldots, k}$, same for $\LM(k,A)$ and $\LM(k,j)$.

Fix $L\in \N$ and write $\sigma(v)_i=\sigma(v_i)$ for $v\in\R^n$, so that $\sigma(v)\in\R^n$. We then define for $x\in\R^{n_0}$ (the input), $\ell\in\gk{1,\ldots,L}$, $b\hk{l}\in \R^{n_{l}}$ (the biases) and $W\hk{l+1}\in\LM\rk{n_l,n_{l+1}}$ (the weights), the neural network as follows
\begin{equation}
    \begin{cases}
        Z\hk{1}(x)&=b\hk{1}+W\hk{1}x\, ,\\
        Z\hk{l+1}(x)&=b\hk{l+1}+W\hk{l+1}\sigma\rk{Z\hk{l}(x)}\, ,
    \end{cases}
\end{equation}
so that $Z\hk{l+1}(x)\in\R^{n_{l+1}}$. Here, $Z\hk{l}(x)$ is the state of the network at layer $l$, given input $x$. See Figure \ref{fig:representationDynamics} for an illustration of the network dynamics for the case $L=3$.
\begin{figure}[h]
    \centering
    \begin{tikzpicture}[
        neuron/.style={circle, draw, minimum size=1cm},
    layer/.style={matrix of nodes, nodes={neuron}, row sep=1cm},
    arrow/.style={-stealth, thick}
]
\pgfmathsetmacro{\numlayers}{5}
\def\layernames{{"$\R^{n_0}$", "$\R^{n_1}$", "$\R^{n_2}$", "$\cdots$", "$\R^{n_3}$"}}

\foreach \name / \y in {1,...,3}
    \node[neuron] (I-\name) at (0,-\y-0.5) {};

\foreach \num in {1,2,3}
{
    \foreach \name / \y in {1,...,4}
        \node[neuron] (H\num-\name) at (\num*2,-\y) {};
}

\foreach \name / \y in {1,...,2}
    \node[neuron] (O-\name) at (\numlayers*2-2,-\y-1) {};
\foreach \source in {1,...,3}
    \foreach \dest in {1,...,4}
        \draw[arrow] (I-\source) -- (H1-\dest);
\foreach \source in {1,...,4}
    \foreach \dest in {1,...,4}
        \draw[arrow] (H1-\source) -- (H2-\dest);
\foreach \source in {1,...,4}
    \foreach \dest in {1,...,4}
        \draw[arrow] (H2-\source) -- (H3-\dest);
\foreach \source in {1,...,4}
    \foreach \dest in {1,...,2}
        \draw[arrow] (H3-\source) -- (O-\dest);

\node[above] at (0,-0.5){\pgfmathparse{\layernames[0]}\pgfmathresult};
\node[above] at (2,-0.5){\pgfmathparse{\layernames[1]}\pgfmathresult};
\node[above] at (4,-0.5){\pgfmathparse{\layernames[2]}\pgfmathresult};
\node[above] at (6,-0.5){\pgfmathparse{\layernames[4]}\pgfmathresult};
\node[above] at (8,-0.5){\pgfmathparse{"$\R^{n_{4}}$"}\pgfmathresult};
\node[below] at (0,-4.6) {$x$};
\draw[arrow] (0.75,-4.8)--(1.25,-4.8) node[midway,above] {$\ssup{b+W\sigma}$};
\node[below] at (2,-4.5) {$Z\hk{1}(x)$};
\draw[arrow] (2.75,-4.8)--(3.25,-4.8) node[midway,above] {$\ssup{b+W\sigma}$};
\node[below] at (4,-4.5) {$Z\hk{2}(x)$};
\draw[arrow] (4.75,-4.8)--(5.25,-4.8) node[midway,above] {$\ssup{b+W\sigma}$};
\node[below] at (6,-4.5) {$Z\hk{3}(x)$};
\draw[arrow] (6.75,-4.8)--(7.25,-4.8) node[midway,above] {$\ssup{b+W\sigma}$};
\node[below] at (8,-4.5) {$Z\hk{4}(x)$};
\end{tikzpicture}
    \caption{A schematic representation of the network for the case $L=3$. Here, the hidden layers ($l\in\gk{1,\ldots,L}$) have the same size, which does not need to be the case for the main theorems. The underlying randomness comes from the $b$'s and the $W$'s}
    \label{fig:representationDynamics}
\end{figure}
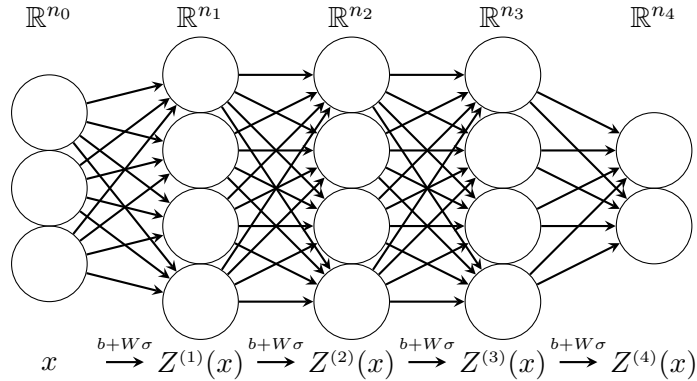
Next, we state the assumption on the biases and weights.
\begin{assumption}
    From now on, we assume that there exists constants $C_b\ge 0$ and $C_W>0$ such that for $l\in\gk{1,\ldots,L}$, $b\hk{l}$ is a standard normal Gaussian vector with mean zero and variance $C_b$. We furthermore assume that for all such $l$, $W\hk{l+1}$ is a matrix with i.i.d. Gaussian entries with mean zero and variance $C_W/n_l$. In addition, we assume that the weights $\rk{W\hk{l}}_l$ and the biases $\rk{b\hk{l}}_l$ are independent between themselves and between layers. Neural networks satisfying this are also called \textnormal{Gaussian} neural networks
\end{assumption}
The above assumption faithfully represents most neural network architecture before training; see \cite{roberts2022principles}. The study of such Gaussian neural networks has inspired many articles. For example, in \cite{hanin2023random} the weak convergence to a Gaussian limit was studied, while in \cite{basteri2024quantitative} quantitative versions of this were produced. This has been applied in several contexts, such as Bayesian inference for Gaussian stochastic processes \cite{lee2017deep}, see \cite{poole2016exponential,jacot2018neural,lee2019wide,macci2024large} for some further examples. 

Next, we state the final assumption, on the joint growth of the layers:
\begin{assumption}
    We assume that $n_1,\ldots,n_L$ jointly tend to infinity with $n$. We furthermore assume that there exists $\hat{l}$ such that for every $l$ in $\gk{1,\ldots,L}$, we have $\lim_{n\to\infty}n_l/n_{\hat{l}}=\gamma_l\in [1,\infty]$. We write $v(n)=n_{\hat{l}}$, the ``slowest" divergence amongst the depth of the layers.
\end{assumption}

Central to proving results (such as \cite{hanin2023random}) for Gaussian neural network is the fact that conditioned on the previous layer, the next layer is a function of a simple Gaussian process, see \cite[Section 1.4]{hanin2023random} for an in-depth discussion. This allows for the application of tools from the theory of Gaussian distributions to prove results. Here, the activation function $\sigma$ is crucial, as it decides how the arising Gaussian processes are transformed.

Our goal is to quantify the atypical behavior of a Gaussian deep neural network, i.e., to prove a large deviation principle (LDP). Our work builds on the \cite{macci2024large}, where this was done for continuous and bounded activation functions $\sigma$. We briefly explain the results and strategy of said article and explain our contribution to the field:\\
In \cite{macci2024large}, the authors used the aforementioned strategy that conditional on the previous layer, the next layer can be described by a transformation of Gaussian random variables. The authors then use the paper \cite{chaganty1997large}, which has as its main result a way to turn a ``conditional" LDP into a full LDP. The assumption that $\sigma$ is bounded and continuous allows the authors in \cite{macci2024large} to conclude the assumptions of \cite{chaganty1997large} are satisfied and one can hence turn the conditional LDP into a full one. We also mention \cite{rhee2019sample} for the large deviations in the presence of heavy tails.

We mention the recent work \cite{hirsch2024large} in which the large deviation analysis of shallow neural networks (i.e., $L=1$) with stochastic gradient descent was conducted. This is a more complicated model and hence seems only tractable for the case $L=1$.
\subsection{Our contributions}
Our main result is the LDP for the deep Gaussian neural network $Z\hk{L}(x)$ in the case of linearly growing activation functions $\sigma$, such as ReLU, see Theorem \ref{thm:main}. This is important because ReLU is one of the activation functions most commonly used in deep neural networks. Our proof is different to \cite{macci2024large}, owing to the faster growth of the activation function: instead of working with the G\"artner--Ellis theorem, we instead apply exponential equivalence and the multidimensional Cram\'er theorem. Indeed, for linearly growing activations $\sigma$, the moment generating function $\kappa$ (see Equation \eqref{eq:kappaDef} for a definition) is no longer finite for all inputs. The G\"artner--Ellis theorem is no longer applicable, as steepness may fail for some activation functions\footnote{Not any commonly used ones (compare Example \ref{exam:}). However, with $\sigma(x)=(x-1)\1\gk{x>0}$, note that for $k(a)=\E_{\Ncal(0,1)}\ek{\ex^{a \sigma(X)^2}}$ is finite for $a\le 1/2$ and that $\lim_{a\to 1/2}k'(a)=\tfrac{1}{\sqrt{2\pi}}\rk{\tfrac{1}{2}+\ex^{1/2}\int_0^\infty \ex^{-x}(x-1)^2\d x}<\infty$.}. This is particularly important when considering extensions to the infinite-dimensional case, see the recent work of \cite{andreis2025ldpcovarianceprocessfully}.

We also obtain a simplification of the rate function obtained in \cite{macci2024large}. There, the rate function was stated as a minimization problem over two families of matrices. We were able to compute the conditional minimizers of the second matrix and hence reduce the complexity of the minimization problem; see Equation \eqref{eq:simpExpression}. Furthermore, in the case of ReLU, we provide an explicit formula of $\kappa$ in terms of a power series, from which one can approximate minimizers, see Equation \eqref{eq:explcitReLU}. This should be helpful in large dimensions, where Monte Carlo becomes less feasible.

We also mention that the case of linearly growing activation functions is considered critical in the following sense: for faster growing functions, the LDP is no longer in the exponential class, as the moment generating function is infinite everywhere, except at the origin, whereas for activation functions growing at most sub-linearly, the moment generating function is finite everywhere which simplifies the analysis.

\subsection{Results}
We need to introduce some additional notation before stating the main result:

We set $\fb{q,p}=\sum_{\alpha,\beta\in A}q_{\alpha,\beta}p_{\alpha,\beta}$, the \textit{Frobenius inner product} between two matrices (or linear maps) $q$ and $p$, with $p,q\in\LM(A,A)$. We write $\nf{q}=\sqrt{\fb{q,q}}$ for the Frobenius norm. Write $\ip{v,w}=\sum_i v_iw_i$ for the usual inner product between vectors $v,w$. Write $\LM_+(A,B)$ for the positive and symmetric linear maps from $A$ to $B$. For $q$ in $\LM_+(A,A)$, write also $q^\#$ for the unique root of $q$, i.e., the unique symmetric and positive definite matrix $q^\#$ with $q=q^\#q^\#$. For $A$ a finite set, we define for $N\in\R^A$, the matrix $\Sigma=\Sigma(q)=\Sigma(q;N)$ by
\begin{equation}
    \Sigma_{\alpha,\beta}(q)=\Sigma_{\alpha,\beta}(q;N)=\sigma\rk{\rk{q^\#N}_\alpha}\sigma\rk{\rk{q^\#N}_\beta}=\sigma\rk{\ip{q^\#_\alpha,N}}\sigma\rk{\ip{q^\#_\beta,N}}\, ,\qquad\textnormal{for }\alpha,\beta \in A\, ,
\end{equation}
where $q^\#_\alpha$ is the $\alpha$-row of $q^\#$. Assume that $N$ is a standard normal Gaussian vector and that $q\in\LM_+(A,B) $. We then define the moment generating function $\kappa\colon \LM(A,A)\times \LM_+(A,A)\to (-\infty,\infty]$
\begin{equation}\label{eq:kappaDef}
    \kappa(\eta;q)=\log\E\ek{\ex^{\fb{\eta,\Sigma(q)}}}\, .
\end{equation}
Note that $\fb{\eta,\Sigma(q)}$ grows at most as quickly as $\sigma(N)^2$ as $N\to\infty$. This implies that $\kappa(\eta;q)$ is finite for all $\eta\in  \LM(A,A)$ if and only if $\sigma$ grows at most sublinearily\footnote{recall that the square of a Gaussian random variable has only finite exponential moments.}, which can be restated as $c_+=0$ in the language of Assumption \ref{ass:activation}. This case was already solved in \cite{macci2024large}.

Let the Legendre transform of $\kappa$ be defined as
\begin{equation}
    \kappa^*(y;q)=\sup_{\eta}\gk{\fb{\eta,y}-\kappa(\eta;q)}\in (-\infty,\infty]\, .
\end{equation}
Before stating the main theorem, we introduce some more notation:

Define for $\ux=(x_\alpha)_{\alpha\in A}$ (with each $x_\alpha\in\R^{n_0}$) and $z$ a linear map $z\colon \R^{n_{L+1}}\to\R^A$, the rate function $ I_{Z,L,\ul{x}}$ as
\begin{equation}\label{def:ratefuntion}
    I_{Z,L,\ul{x}}(z)=\inf_{g\hk{L}\in\LM_{+}, r\in M(n_{L+1},A)}\gk{I_{G,L,\ul{x}}\rk{g\hk{L}}+\frac{\fb{r,r}}{2}\colon g^{\ssup{L},\#}r=z}\, ,
\end{equation}
where
\begin{equation}\label{eq:defIg}
    I_{G,L,\ul{x}}\rk{g\hk{L}}=\inf\gk{\sum_{l=1}^{L}J(g\hk{l}|g\hk{l-1})\colon g\hk{0},g\hk{1},\ldots,g\hk{L-1}\in\LM_{+}}\, ,
\end{equation}
with $g\hk{0}=\rk{C_b+\frac{C_w}{n_0}\lk{x_\alpha,x_\beta}}_{\alpha,\beta\in A}$ and
\begin{equation}
    J(g\hk{l}|g\hk{l-1})=\gamma_l\kappa^*\rk{\frac{g\hk{l}-C_b\1}{C_W};g\hk{l-1}}\, ,
\end{equation}
following the usual rule that $\infty*0=0$. The function $I_{Z,L,\ul{x}}$ was introduced in \cite{macci2024large}. The main result of the paper is a LDP for the rescaled distribution of the family of outputs
\begin{equation}
    \begin{pmatrix}
Z\hk{L+1}(x_1)^T \\
 \ldots\\
 Z\hk{L+1}(x_{|A|})^T
\end{pmatrix}=    \begin{pmatrix}
Z\hk{L+1}(x_1)_1 &\ldots & Z\hk{L+1}(x_1)_{n_{L+1}}\\
 & \ldots &\\
 Z\hk{L+1}(x_{|A|})_1 &\ldots & Z\hk{L+1}(x_{|A|})_{n_{L+1}}
\end{pmatrix}\, ,
\end{equation}
where we have enumerated $(x_\alpha)_{\alpha\in A}$ as $(x_i)_{i=1,\ldots, |A|}$.
\begin{theorem}\label{thm:main}
    Fix $A$ a finite set and fix points $x_\alpha\in\R^{n_0}$ with $\alpha\in A$. The random vector $\rk{Z\hk{L+1}(x_\alpha)}_{\alpha\in A}$ then induces the linear map\footnote{Alternatively, we can interpret $\uZx$ as a random element in $\R^A\otimes\R^{n_{L+1}}$, with $\uZx=\sum_{\alpha\in A}e_\alpha\otimes Z\hk{L+1}(x_\alpha)$, where $e_\alpha$ is the basis vector corresponding to $\alpha\in A$. This generalizes well for $\abs{A}=\infty$.} $\uZx\colon\R^{n_{L+1}}\to \R^A$ with $\uZx( v)=\rk{\lk{Z\hk{L+1}(x_\alpha),v}}_{\alpha\in A}$. Then, $\rk{\frac{\uZx}{\sqrt{v(n)}}}_n$ satisfies an LDP at speed $v(n)$ with good rate function $I_{Z,L,\ul{x}}(z)$.

    Furthermore, we have the simplified expression
    \begin{equation}\label{eq:simpExpression}
        I_{Z,L,\ul{x}}(z)=\inf_{\mycom{g\hk{L}\in\LM_{+}}{\mathrm{Im}(g)\supset \mathrm{Im}(z)}}\gk{I_{G,L,\ul{x}}\rk{g\hk{L}}+\frac{\nf{\rk{g\hok{L,\#}}^+z}^2}{2}}\, ,
    \end{equation}
    where the superscript $+$ is the Moore--Prenrose inverse of a matrix, see \cite{Penrose_1955}.
\end{theorem}
Note that using the Dawson--G\"artner theorem (see \cite[Theorem 4.6.1]{dembo2009large}) one can get a large deviation theorem in the product topology for the case of infinitely many training examples. A functional LDP is work in progress.

We furthermore have an explicit expansion of $\kappa$ in the case of ReLU activation. As it needs several additional definitions, we have decided to state it explicitly in the appendix only.
\begin{lemma}
    For $\sigma(x)=x\1\gk{x>0}$, we have that $\kappa(\eta;q)$ has the power-series expansion given by Equation \eqref{eq:explcitReLU}.
\end{lemma}
See Figure \ref{FigLast} for an illustration in the $n_0=1$, ReLU case.

\textbf{Organization of the paper}:
\begin{itemize}
    \item In Section \ref{sec:lapalcefunctiom}, we give several properties of $\kappa$ and its conjugate $\kappa^*$. We also restate the result from \cite{macci2024large} which allows for a convenient reparameterization of the problem.
    \item In Section \ref{sec:onelayer}, we prove the main theorem. Our main tools include the classical large deviation theory, such as exponential approximations and the multidimensional Cram\'er theorem in conjunction with \cite{chaganty1997large}.
    \item In Section \ref{sec:analysis}, we simplify the expression for the rate function from \cite{macci2024large}.
\end{itemize}
\section{Proof}

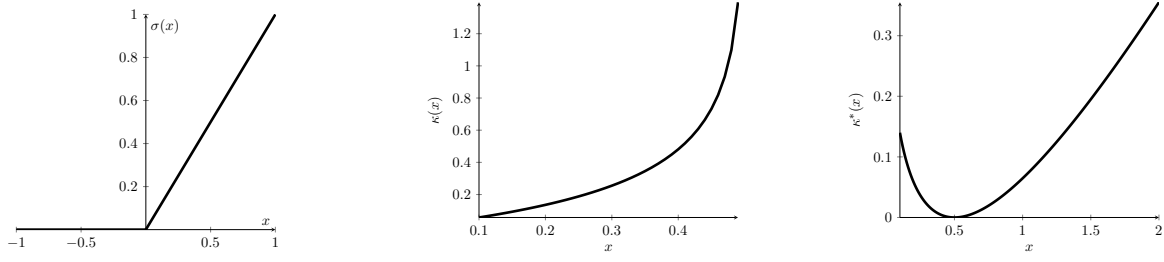
\begin{figure}
\begin{minipage}{.32\textwidth}
                  \begin{tikzpicture}[scale=0.5]
        \centering 
    \begin{axis}[
        axis lines = center,
        xlabel = {$x$},
        ylabel = {$\sigma(x)$},
    ]
    \addplot[ line width=2pt,
    ] table [
        col sep=space,
        x index=0,
        y index=1
    ] {OldTexFiles/data0.csv};
    \end{axis}
\end{tikzpicture}
 \end{minipage} \begin{minipage}{.32\textwidth}
                  \begin{tikzpicture}[scale=0.5]
        \centering 
    \begin{axis}[
        axis lines = left,
        xlabel = {$x$},
        ylabel = {$\kappa(x)$},
    ]
    \addplot[line width=2pt,
    ] table [
        col sep=space,
        x index=0,
        y index=1
    ] {OldTexFiles/data.csv};
    \end{axis}
\end{tikzpicture}
 \end{minipage}
  \begin{minipage}{.32\textwidth}
                  \begin{tikzpicture}[scale=0.5]
        \centering 
    \begin{axis}[
        axis lines = left,
        xlabel = {$x$},
        ylabel = {$\kappa^*(x)$},
    ]
    \addplot[ line width=2pt,
    ] table [
        col sep=space,
        x index=0,
        y index=1
    ] {OldTexFiles/data1.csv};
    \end{axis}
\end{tikzpicture}
 \end{minipage}\caption{\textsc{Left:} ReLU $\sigma(x)=x\1\gk{x\ge 0}$, \textsc{middle:} $\kappa$ for ReLU and $q=1$, \textsc{right:} $\kappa^*$ for ReLU and $q=1$.}
  \label{FigLast}
\end{figure}

\subsection{Analysis of the Laplace transform}\label{sec:lapalcefunctiom}
Recall the definitions of $\kappa$ and $\kappa^*$ from the previous section. 

Write $\Dcal=\Dcal_q=\gk{\eta\colon \kappa(\eta;q)\in\R}$ domains of finiteness for $\kappa$. Next, we discuss the dependence of $\Dcal$ and $\Dcal^*$ on $q$.
\begin{lemma}\label{lem:domainskappa}
The following hold true
\begin{enumerate}
    \item Choose $C>0$ such that $\abs{\sigma(x)}\le C+C\abs{x}$. Then,
    \begin{equation}
        \gk{\eta \colon \nf{\eta}<\rk{C \nf{q^\#}2}^{-2}}\subset \Dcal_q\, .
    \end{equation}
    \item The map $q\mapsto q^\#$ is continuous. 
\end{enumerate}
\end{lemma}

\begin{proof}
     To prove the first claim, choose $C>0$ with $\abs{\sigma(x)}\le C+ C \abs{x}$ for all $x\in\R$, which is possible due to Assumption \ref{ass:activation}. We can then bound
     \begin{equation}
         \abs{\sigma(ax)\sigma(bx)}\le \rk{C+C\abs{a}{x}}\rk{C+C\abs{b}\abs{x}}\le 2C^2\rk{1+\max\gk{a,b}^2\abs{x}^2}\, .
     \end{equation}
     By the definition of the Frobenius norm, we have that $\abs{q^\#N_\alpha}\le \nf{q^\#}\abs{N}$. We hence can bound $\Sigma(q)$ as
\begin{equation}\label{eq:SigmaBound}
    \Sigma_{\alpha,\beta}(q;N)\le 2 C^2\rk{1+\nf{q^\#}^2\sum_{\alpha}N_\alpha^2}\, .
\end{equation}
Therefore,
\begin{equation}
    \fb{\eta,\Sigma(q)}\le 2 C^2\rk{1+\nf{q^\#}^2\sum_{\alpha}N_\alpha^2}\sum_{\alpha,\beta}\abs{\eta_{\alpha,\beta}}\le  2 C^2\rk{1+\nf{q^\#}^2\sum_{\alpha}N_\alpha^2}\nf{\eta}\, .
\end{equation}
    Recall that $\E_{\Ncal(0,1)}\ek{\ex^{aX^2}}<\infty$ if $a<1/2$. Hence, recalling Equation \eqref{eq:kappaDef}, if $\nf{\eta}<\rk{C \nf{q^\#} 2}^{-2}$, $\kappa(\eta;q)$ is finite.\\
    Regarding the second claim, this follows from the singular value decomposition, see the \cite[Theorem 7.2.6]{horn2012matrix}, as we can write $q=UDU^T$ (with $U$ orthogonal and $D$ diagonal) and $q^\#=UD^{1/2}U^T$ (where $D^{1/2}$ as the square-root of the values of $D$ on the diagonal). The continuity of the map $q\mapsto q^\#$ then follows from the continuity of the square-root map on $[0,\infty)$.
\end{proof}
\begin{lemma}\label{lem:continuity}
    Given $q_n\to q$, we have that for all $\eta$
    \begin{equation}
        \lim_{n\to\infty}\kappa(\eta;q_n-q)=0\, .
    \end{equation}
    Furthermore\footnote{In an earlier version of this work, the proof did not cover all cases of activation}, if $\eta$ is in the interior of $\Dcal_q$, then 
    \begin{equation}
        \lim_{q_n\to q}\kappa(\eta,q_n)=\kappa(\eta,q)\, .
    \end{equation}
\end{lemma}
\begin{proof}
    For the first statement, we need to show that
    \begin{equation}
        \lim_{n\to\infty}\E\ek{\ex^{\fb{\eta,\Sigma(q)-\Sigma(q_n)}}}=1\, .
    \end{equation}
    Using the almost everywhere continuity of $\sigma$ (and hence of $\fb{\eta,\Sigma(q)}$), given dominated convergence, the claim follows as soon as we can prove that for all $\e>0$, there exists $N>0$ large enough such that $n\ge 0$
    \begin{equation}\label{eq:72520251}
        \fb{\eta,\Sigma(q)-\Sigma(q_n)}\le \e \nf{\eta}\norm{N}_2^2+C\, .
    \end{equation}
    Indeed, this is an integrable majorant given $\e\nf{\eta}<1/2$.

    Recall that $\abs{\sigma(x)}\le c_+\abs{x}(1+o(x))$ as $x\to\infty$ and is bounded on balls around the origin. Hence we can can write
    \begin{equation}
        \abs{\sigma\rk{q_n^\#N_\alpha}-\sigma\rk{q^\#N_\alpha}}\le 2c_+ \abs{(q_n^\#-q^\#)N_\alpha}+C\le 2c_+ \nf{(q_n^\#-q^\#)} \norm{N_\alpha}+C\, .
    \end{equation}
    Equation \eqref{eq:72520251} now follows recalling the continuity of the $A\mapsto A^\#$ map and by expanding
    \begin{multline}
\abs{\sigma\rk{q^\#N_\alpha}\sigma\rk{q^\#N_\beta}-\sigma\rk{q^\#_nN_\alpha}\sigma\rk{q^\#_nN_\beta}}=\abs{\sigma\rk{q^\#N_\alpha}-\sigma\rk{q^\#_nN_\alpha}}\abs{\sigma\rk{q^\#N_\beta}}\\
+\abs{\sigma\rk{q^\#N_\beta}-\sigma\rk{q^\#_nN_\beta}}\abs{\sigma\rk{q^\#_nN_\alpha}}\, .
    \end{multline}
    This concludes the proof of the first statement.

    Next we prove the second statement. Note that we can assume that $c_+>0$ and that hence $\sigma(x)=c_+x+o(x)$ as $x\to\infty$, as the statement trivially follows from dominated convergence otherwise. Note that for any $\e>0$ there exists $n_0\ge 1$ large enough such that for all $N\in\R^d$, $\alpha\in A$ and $n\ge n_0$
    \begin{equation}\label{eq.du}
        \abs{q_n^\#N_\alpha-q^\#N_\alpha}\le \e \norm{N}/2\, .
    \end{equation}
    There exists $K>0$ such that for all $x$ with $\abs{x}>K$ it holds that
    \begin{equation}\label{eq.fu}
        \abs{\sigma\rk{x}-c_+x\1\gk{x>0}-c_-x\1\gk{x<0}}\le \e \abs{x}/2\, .
    \end{equation}
    Hence
    \begin{equation}
        \abs{\sigma\rk{q_n^\#N_\alpha}-\sigma\rk{q^\#N_\alpha}}\le \e(1+c_+)\norm{N}+C_\sigma\, ,
    \end{equation}
    for $C_\sigma>0$ some large constant, as $\sigma$ might not be continuous inside $\ek{-K,K}$. Indeed, if $\sigma(q^\#N_\alpha)=0$, then the error term from Eq.~\eqref{eq.du} dominates, whereas if it is not zero, then Eq.~\eqref{eq.fu} dominates. Thus for some $f_{\alpha,\beta}(q,n)$ (and by adjusting $C_\sigma$)
    \begin{equation}
\sigma(q_n^\#N_\alpha)\sigma(q_n^\#N_\beta)=\sigma(q^\#N_\alpha)\sigma(q^\#N_\beta)+f_{\alpha,\beta}(q,n)\, ,
    \end{equation}
with
\begin{equation}
    \abs{f_{\alpha,\beta}(q,n)}\le C \e \norm{N}^2+C_\sigma\, ,
\end{equation}
for $C=C(c_+,q)$. Hence for
\begin{equation}
    \fb{\eta,\Sigma(q_n)}-\fb{\eta,\Sigma(q)}=h_n(\eta,q,n,N)\, ,
\end{equation}
we obtain $\abs{h_n}\le  C \e \norm{N}^2\nf{\eta}+C_\sigma$ (again adjusting $C_\sigma$). Therefore
\begin{equation}
    \E\ek{\ex^{\fb{\eta,\Sigma(q_n)}}}=\E\ek{\ex^{\fb{\eta,\Sigma(q)}+h_n(\eta,q,n,N)}}\, .
\end{equation}
By Young's inequality
\begin{equation}
    \ex^{\fb{\eta,\Sigma(q)}+h_n(\eta,q,n,N)}\le \frac{\ex^{\fb{(1+\delta)\eta,\Sigma(q)}}}{1+\delta}+\frac{\ex^{C\gamma \e \norm{N}^2\nf{\eta}+C_\sigma}}{\gamma}\, ,
\end{equation}
with $\frac{1}{1+\delta}+\frac{1}{\gamma}=1$ and $\gamma>1$ and $\delta>0$. Choosing $\delta>0$ small enough, the first term is integrable ($\eta$ is in the interior!) and for $\e>0$ small enough, the same is true for the second term. The result then follows by dominated convergence since $\sigma$ is continuous on a set of measure 1.
\end{proof}
Next we recall the representation from \cite{macci2024large}: we define for $l\ge 1$ a recursive family of matrix-valued functions
\begin{equation}
    G_{n}\hk{l}(\ul{x})=C_b+\frac{C_W}{n_l}\sum_{j=1}^{n_l}\Sigma\rk{G_n\hk{l-1}(\ul{x});N_j\hk{l}}\, ,
\end{equation}
where 
\begin{equation}
    G_n\hk{0}(\ul{x})=g\hk{0}(\ul{x})=\rk{C_b+\frac{C_W}{n_0}\lk{x_\alpha,x_\beta}}_{\alpha,\beta\in A}\, ,
\end{equation}
and $\gk{N_j\hk{l}}_{j,l\ge 1}$ a family of i.i.d. standard Gaussian vectors in $\R^A$.
\begin{lemma}{\cite[Lemma 3.3]{macci2024large}}\label{lem:rep}
The random variables
    \begin{equation}\label{eq:gLDP}
    \rk{G_{n}^{(\ssup{L}),\#}(\ul{x})N_{h}\hk{L+1}}_{h=1,\ldots,n_{L+1}}=\rk{\rk{\sum_{\gamma\in A}G_{n,\alpha \gamma}^{(\ssup{L}),\#}\rk{\ul{x}}\rk{N_{h}\hk{L+1}}_\gamma}_{\alpha\in A}}_{h=1,\ldots,n_{L+1}}\, ,
\end{equation}
have the same distribution as $\uZx$. Furthermore, if $G_{n}^{(\ssup{L}),\#}\rk{\ul{x}}$ satisfies an LDP with rate function $ I_{G,L,\ul{x}}$ (see Equation \eqref{eq:defIg} for the definition), then $\uZx$ satisfies a LDP with rate function $I_{Z,L,\ul{x}}(z)$.
\end{lemma}
\begin{proof}
We recap the proof for the reader's convenience: define the function $f\colon M_+\rk{{A},{A}}\times M\rk{{A}, {n_2}}\to M\rk{{A}, n_{2}}$ of two matrices as follows: $f(A,B)=A^\# B$. Then $f$ is continuous on its domain by Lemma \ref{lem:domainskappa}. Furthermore, $\frac{N\hk{L+1}}{\sqrt{v(n)}}$ satisfies a LDP on $M\rk{A, n_2}$ with rate function $f(r)=\fb{r,r}/2=\nf{r}^2/2$ (see e.g. \cite[Exercise 2.2.23]{dembo2009large}). Hence, by the continuous mapping principle (see \cite[Theorem 4.2.1]{dembo2009large}), the result follows.
\end{proof}
\subsection{Inductive large deviations}\label{sec:onelayer}
\begin{lemma}\label{lem:oneLayer}
    If $L=1$, $ Z\hk{2}(x)/\sqrt{v(n)}$ satisfies an LDP with good rate function $I_{Z,2,\ul{x}}(z)$.
\end{lemma}
\begin{proof}
By Lemma~\ref{lem:rep}, it suffices to prove a LDP for $\rk{G_n\hk{1}}_n$. Fix $\ux=(x_\alpha)_{\alpha\in A}$ our input.

Define $Y_j=C_b+C_W\Sigma\rk{g\hk{0}(\ul{x}),N_j\hk{1}}$, for $j=1,\ldots,n_1$. This is an i.i.d. family of random variables with log-moment generating function given by 
\begin{equation}
    \log \E\ek{\ex^{\fb{\eta,Y_j}}}=\fb{\eta,C_b\1}+\kappa\rk{C_W\eta;g\hk{0}(\ul{x})}=:\Psi\rk{\eta;g\hk{0}(\ul{x})}\, .
\end{equation}
It is finite for $\eta\in\Dcal_{g\hk{0}(\ul{x})}\supset \gk{\eta\colon \nf{\eta}<\e}$, by Lemma \ref{lem:domainskappa}. By the multidimensional Cram\'er theorem \cite[p. 61]{rassoul2015course}, we have that $G_n\hk{1}(\ul{x})=\frac{1}{n_1}\sum_{j=1}^{n_1}Y_j$ satisfies an LDP with good rate function $\Psi^*\rk{\xi;g\hk{0}(\ul{x})}=\kappa^*\rk{\frac{\xi-C_b\1}{C_W};g\hk{0}(\ux)}$. The result then follows from Lemma \ref{lem:rep}.
\end{proof}
Let us recall the main result of \cite{chaganty1997large}:
\begin{definition}\label{def:LDPcont}
    Let $\rk{\Omega_i,\Bcal_i}_{i=1,2}$ be two Polish spaces with associated Borel sigma-algebra. A sequence of transition kernels $\gk{\nu_n\colon  \Omega_1\times\Bcal_2\to [0,1]}_n$ is said to satisfy the \textnormal{LDP continuity condition} with rate function $J$ if
\begin{enumerate}
\item For each $x\in \Omega_1$, $J(x_1,\cdot)$ is a good rate function.
    \item For each $x\in\Omega_1$ and each sequence $x_n\to x$, we have that $\gk{\nu_n(x_n,\cdot)}_n$ satisfies an LDP on $\Omega_2$ with rate function $J(x,\cdot)$.
    \item $(x_1,x_2)\mapsto J(x_1,x_2)$ is lower-semi-continuous.
\end{enumerate}
\end{definition}

We then have that:
\begin{theorem}{\cite[Theorem 2.3]{chaganty1997large}}\label{thm:CondLDP}
    Let $\gk{\mu_n}_n$ be a sequence of probability measures on $\Omega_1$, satisfying an LDP with good rate function $I_1$. Suppose that $\gk{\nu_n}_n$ satisfies the LDP continuity condition with rate function $J$. Then
    \begin{enumerate}
        \item Then, $\xi_n$ defined by $\xi_n(A\times B)=\int_A \nu_n(x,B) \d \mu_n(x)$ satisfies a \textnormal{weak} LDP with rate function $I\colon (x_1,x_2)\mapsto I_1(x_1)+J(x_1,x_2)$. 
        \item If $I=I_1+J$ is a good rate function, the $\xi_n$ satisfies a LDP.
        \item $\xi\hk{2}_n(B)=\int_{\Omega_1}\nu_n(x,B)\d \mu_n(x)$ satisfies an LDP with rate function $I_2=\inf_{x_1}\gk{I_1+J}$ (which is good if $I$ is good).
    \end{enumerate}
\end{theorem}
\textbf{Proof of Theorem \ref{thm:main}}: by Lemma \ref{lem:rep}, it suffices to prove a LDP for $G_n\hk{L}(\ux)$.

We use the ``inductive'' LDP approach from \cite{macci2024large}. We have finished the case $L=1$ in Lemma \ref{lem:oneLayer}. Take now $L>1$ and assume that the theorem has been proved for $\gk{1,\ldots,L-1}$. We furthermore assume that $\gamma_l<\infty$ for all $l$, as the other case is true with less assumptions (see \cite{macci2024large} for details). To prove Theorem \ref{thm:main}, we need to verify the conditions from Theorem \ref{thm:CondLDP} (with the identification $J(q,y)=\Psi^*\rk{y,q}=\kappa^*\rk{\frac{y-C_b\1}{C_W};q}$ and $I_1=I_{G,L-1,\ul{x}}$). We begin with the second condition:

\textbf{(2)} Fix $g\hk{L-1}$. Take $\gk{g\hk{L-1}_n}_n$ converging to $g\hk{L-1}$.\\
We begin by proving exponential equivalence (see \cite[Definition 4.2.10]{dembo2009large}) between the law of $G_n\hk{L}(\ux)$ conditioned on $g\hk{L-1}_n$ and conditioned on $g\hk{L-1}$. By the exponential Chebyshev inequality, we get for $t>0$
\begin{equation}
    \P\rk{G_n\hk{L}(\ux)|g\hk{L-1}-G_n\hk{L}(\ux)|g\hk{L-1}_n>\delta }\le \ex^{-\delta n t}\E\ek{\ex^{C_Wt\ek{\Sigma\rk{g\hk{L-1};N_1\hk{l}}-\Sigma\rk{g\hk{L-1}_n;N_1\hk{l}}}}}^n.
\end{equation}
By Lemma \ref{lem:continuity}, for any $t>0$, we can bound
\begin{equation}
    \limsup_{n\to\infty}\frac{1}{n}\log\P\rk{G_n\hk{L}(\ux)|g\hk{L-1}-G_n\hk{L}(\ux)|g\hk{L-1}_n>\delta }\le-\delta t\, .
\end{equation}
The same bound also holds for $G_n\hk{L}(\ux)|g\hk{L-1}_n-G_n\hk{L}(\ux)|g\hk{L-1}>\delta$. Taking $t\to \infty$ proves exponential equivalence.

The exponential equivalence now implies (\cite[Theorem 4.2.13]{dembo2009large}) that the LDP of $G_n\hk{L}(\ux)$ conditioned on $g\hk{L-1}_n$ follows from the LDP of $G_n\hk{L}(\ux)$ conditioned on $g\hk{L-1}$. However, the LDP for $G_n\hk{L}(\ux)$ conditioned on $g\hk{L-1}$ follows from the same argument as in Lemma \ref{lem:oneLayer}, i.e., the LDP now follows from the multidimensional Cram\'er theorem in \cite[p. 61]{rassoul2015course}.

\textbf{(3)} Take a sequence $\rk{a_n,b_n}_n$ converging to $(a,b)$. We then need to show that $\liminf_{n\to\infty} J(a_n|b_n)\ge J(a,b)$. For simplicity, take $C_b=0$ and $C_W=1$, as these are only constant shifts. We expand
\begin{equation}
   \liminf_{n\to\infty} J(a_n|b_n)\ge  \liminf_{n\to\infty}\fb{\eta,a_n}-\kappa\rk{\eta;b_n}\ge \fb{\eta,a}-\limsup_{n\to\infty}\kappa\rk{\eta;b_n}\, .
\end{equation}
If $\eta\in \Dcal_{b}^\circ$, we obtain by Lemma \ref{lem:continuity}, that $\limsup_{n\to\infty}\kappa\rk{\eta;b_n}=\kappa(\eta;b)$ and hence get (by taking the $\sup$ over all $\eta\in\Dcal_b^\circ$)
\begin{equation}
    \liminf_{n\to\infty} J(a_n|b_n)\ge \sup_{\eta\in\Dcal_b^\circ}\gk{\fb{\eta,a}-\kappa\rk{\eta;b}}\, .
\end{equation}
It remains to show that we can replace $\Dcal_b^\circ$ by $\Dcal_b$ in the above equation. Take $\eta\in \partial \Dcal_b$ and assume that $\kappa\rk{\eta;b}<\infty$, as otherwise there is nothing to show. It remains to show that there exists $(\eta_n)_n\in\Dcal_q^\circ$ such that
\begin{equation}
    \liminf_{n\to\infty}\kappa\rk{\eta_n;b}\le\kappa(\eta;b)\, .
\end{equation}
However, convexity together with $\kappa(0;b)=0$ imply (note that the domain of $\kappa$ is convex due to the H\"older inequality, see \cite[Theorem 4.24.]{rassoul2015course})
\begin{equation}
    \kappa\rk{(1-1/n)\eta+1/n\cdot 0;b)}\le (1-1/n)\kappa(\eta;b)+1/n\kappa(0;b)=(1-1/n)\kappa(\eta;b)\, .
\end{equation}
Taking the $\liminf$ as $n\to\infty$ concludes the proof.

\textbf{(1)} It remains to show the compactness of level-sets. By \cite[Lemma 2.6]{chaganty1997large}, it suffices to show that for every $a\ge 0$ and every $K\subset \LM(A,A)$ compact, that
\begin{equation}
    \bigcup_{\xi\in K}\gk{\zeta\in \LM(A,A)\colon J(\zeta|\xi)\le a}\, ,
\end{equation}
is compact. We generalize the argument from \cite[Lemma 2.2.20]{dembo2009large}: take a sequence $\rk{\xi_n,\zeta_n}_n$ and assume w.l.o.g. that $\rk{\xi_n}_n$ converges to some $\xi_\infty$. Fix now $\e>0$ such that $\kappa(\eta; \xi_n)<\e^{-1}$ for all $\eta\in B_\e(0)$ and $n\in \N\cup\gk{+\infty}$. This is possible by Lemma \ref{lem:domainskappa} and potentially removing some initial members of the sequence $\rk{\xi_n}_n$. Hence, we get for $\eta\in B_\e(0)$
\begin{equation}
    \frac{J(\zeta|\xi_n)}{\nf{\zeta}}=\frac{\sup_{\tilde{\eta}}\gk{\fb{\tilde{\eta},{\zeta}}-\Psi_L\rk{\tilde{\eta};\xi_n}}}{\nf{\zeta}}\ge \fb{\eta,\frac{\zeta}{\nf{\zeta}}}-\frac{\Psi_L\rk{\eta;\xi_n}}{\nf{\zeta}}\, .
\end{equation}
Now, as $\nf{\zeta}\to\infty$, the last term on the right-hand side goes to zero uniformly in $n\in \N\cup\gk{+\infty}$. Observe that the first term can be made uniformly positive (by possibly rotating $\eta$). Hence, $J(\zeta|\xi_n)$ goes to $+\infty $ as $\nf{\zeta}\to\infty$, uniformly in $n\in \N\cup\gk{+\infty}$. This implies compactness of level sets, and hence $\rk{\zeta_n}_n$ has a converging subsequence. This concludes the proof.

This establishes the LDP for $G_n\hk{L}(\ux)$ and by Lemma \ref{lem:rep} concludes the proof of the main theorem.
\subsection{Analysis of the rate function}\label{sec:analysis}
\begin{proposition}
    We have that for $g\in \LM_{+}$ fixed, that the minimization problem in Equation \eqref{def:ratefuntion} has a unique solution given by the following linear map $F(z,g)=\rk{g\hok{L,\#}}^+z$, where $q^+$ is the Moore--Penrose inverse of a matrix $q$.
\end{proposition}
\begin{proof}
    Write $g$ instead of $g\hok{L,\#}$ for this proof. Write $n=n_{L+1}$ and $\abs{A}=m$. 

    Using a standard Lagrangian approach, it holds that to solve Equation \eqref{def:ratefuntion}, $r$ has to solve the following two equations
    \begin{equation}\label{Eq:lagrange}
        gr=z\qquad\text{and}\qquad g\mu=r\, ,
    \end{equation}
    for some $\mu\in\R^{m\times n}$ the Lagrange multiplier. Indeed, the second equation comes from the derivative of $\nf{r}^2/2$. Equation \eqref{Eq:lagrange} is a linear system with $2nm$ equations and hence must have at least one solution. Notice that any generalized inverse $g^+$ satisfies the first equation because $gg^+w=w$ for any $w$ in the range of $g$. However, the range of $g$ must contain the range of $z$, for the problem to be well-posed. If the range of $g^+z$ was contained in the range of $g$, we could choose $\mu=g^+g^+$, because then we have $g\mu=gg^+g^+z=2g^+z$. However, for $g^+$ chosen to be the Moore--Penrose inverse of $g$, we have that the range of $g^+$ is the range of $g^T$, see \cite{Penrose_1955}. However, as $g$ is symmetric, we have that the two ranges agree.
\end{proof}
\section{A power-series expansion for ReLU}
In this section, we prove a power series expansion for $\kappa(\eta;q)$ in terms of products over Gamma functions. As before, we can assume that $q$ is diagonal with positive values on the diagonal.
\begin{lemma}
We have for $a>0$ and $k\in \N\cup\gk{0}$ that
\begin{equation}
   \frac{\sqrt{a}}{\sqrt{2\pi}} \int_0^\infty x^k\ex^{-ax^2/2}\d x=a^{-k/2}2^{k/2}\Gamma\rk{\frac{k+1}{2}}\, .
\end{equation}
\end{lemma}
\begin{proof}
    This follows from a change of variables $x\mapsto \sqrt{x}$ and the definition of the Gamma function.
\end{proof}
As an immediate consequence, take $\beta\in \N_0^n$ and $a=(a_1,\ldots,a_n)\in (0,\infty)^n$. We then have that
\begin{equation}\label{Eq:app:gamma}
    \frac{\sqrt{a_1\cdots a_n}}{(2\pi)^{n/2}}\int_{\R^n}\rk{\prod_{i=1}^n \rk{x_i\1\gk{x_i>0}}^{\beta_i}}\ex^{-\sum_{i=1}^na_ix_i^2/2}\d x=\prod_{\mycom{i=1}{\alpha_i\neq 0}}^n\rk{a_i^{-\beta_i/2}2^{\beta_i/2}\Gamma\rk{\frac{\beta_i+1}{2}}}=:R(\beta)\, ,
\end{equation}
where $a$ is an implicit argument of $R$.
Assume now that $q$ diagonal with $a$ on the diagonals. For $\ul{\alpha}\in\N^k$ and $\ul{\beta}\in\N^k$ and $i\in \gk{1,\ldots,n}$, write $\abs{\ul{\alpha},\ul{\beta}}_i=\#\gk{l\colon \alpha_l=i\text{ or }\beta_l=i}$, i.e., the number of times that $i$ appears in the joint vector $(\ul{\alpha},\ul{\beta})$. We then have that
\begin{equation}\label{eq:InitialExpansion}
    \E\ek{\ex^{\fb{\eta,\Sigma}}}=\sum_{k\ge 0}\frac{1}{k!}\sum_{\ul{\alpha},\ul{\beta}\in [n]^k}R\rk{\abs{\ul{\alpha},\ul{\beta}}_i}\prod_{l=1}^k\eta_{\ul{\alpha}_l,\ul{\beta}_l}\, .
\end{equation}
Indeed, this follows immediately from the Taylor-expansion of the exponential
\begin{equation}
    \E\ek{\ex^{\fb{\eta,\Sigma}}}=\sum_{k\ge 0}\frac{1}{k!}\E\ek{\fb{\eta,\Sigma}^k}\, ,
\end{equation}
and the explicit formula from Equation \eqref{Eq:app:gamma}. Note that that the product over $l\in \gk{1,\ldots,k}$ can be alternatively written as $\eta^{\otimes k}_{\ul{\alpha},\ul{\beta}}$. The formula in Equation \eqref{eq:InitialExpansion} converges exponentially fast and so does
\begin{equation}\label{eq:explcitReLU}
     \kappa(\eta;q)=\log\E\ek{\ex^{\fb{\eta,\Sigma}}}=\log\rk{\sum_{k\ge 0}\frac{1}{k!}\sum_{\ul{\alpha},\ul{\beta}\in [n]^k}R\rk{\abs{\ul{\alpha},\ul{\beta}}_i}\prod_{l=1}^k\eta_{\ul{\alpha}_l,\ul{\beta}_l}}
\end{equation}
For convenience, we give the first-order expansion in $\nf{\eta}$
\begin{equation}
    \kappa(\eta;q)=2\sum_{\alpha\neq\beta}\eta_{\alpha,\beta}a_\alpha^{-1/2}a_\beta^{-1/2}+\sqrt{\pi}\sum_{\alpha}\eta_{\alpha,\alpha}a^{-1}_\alpha+o\rk{\nf{\eta}}\\, ,
\end{equation}
which readily follows from $\Gamma(3/2)=\sqrt{\pi}/2$.
\bibliographystyle{alpha}
\bibliography{thoughts.bib}
\end{document}